\documentclass[letterpaper, 10 pt, journal, twoside]{IEEEtran}
\ifCLASSINFOpdf
\else
\fi
\usepackage[utf8]{inputenc}

\usepackage{amsmath,amsthm}
\usepackage{amssymb,nicefrac}
\usepackage{color}
\usepackage{cite}
\usepackage[small]{caption}
\usepackage{subcaption}

\newcommand{\blue}[1]{\textcolor{black}{#1}}

\usepackage[ruled,vlined,linesnumbered]{algorithm2e}
\usepackage{algorithmic,algorithm2e,float}

\SetKw{Continue}{continue}

\usepackage{graphicx}
\newcommand{\be}{\begin{equation}}
\newcommand{\ee}{\end{equation}}
\newcommand{\vect}[1]{\boldsymbol{#1}}
\newcommand{\vphi}{\vect{\phi}}

\newcommand{\Tcal}{\mathcal{T}}
\newcommand{\T}{\mathcal{T}}
\newcommand{\R}{\mathcal{R}}
\newcommand{\E}{\mathbb{E}}

\newcommand{\w}{\vect{w}}
\newcommand{\vpsi}{\vect{\psi}}

\DeclareMathOperator*{\argmax}{arg\,max}

\newcommand{\new}{\mathtt{new}}
\newcommand{\curr}{\mathtt{curr}}
\newcommand{\wnew}{\w^{\new}}
\newcommand{\wcurr}{\w^{\curr}}
\newcommand{\Tnew}{\T^{\new}}
\newcommand{\Tcurr}{\T^{\curr}}

\newcommand\oprocendsymbol{\hbox{$\bullet$}}
\newcommand\oprocend{\relax\ifmmode\else\unskip\hfill\fi\oprocendsymbol}

\newcommand{\bv}{\begin{bmatrix}}
\newcommand{\ev}{\end{bmatrix}}

\theoremstyle{definition}
\newtheorem{problem}{Problem}
\newtheorem{theorem}{Theorem}

\newtheorem{definition}{Definition}
\newtheorem{proposition}{Proposition}
\newtheorem{observation}{Observation}
\newtheorem{example}{Example}

\newtheorem{lemma}{Lemma}
\theoremstyle{remark}
\newtheorem{remark}[theorem]{Remark}

\begin{document}
%
\title{Learning Submodular Objectives for Team Environmental Monitoring
}
%
%
%

\author{Nils Wilde, Armin Sadeghi, and Stephen L.~Smith%
\thanks{Manuscript received: September, 10, 2021; Revised November, 27, 2021; Accepted November, 29, 2021.}
\thanks{ This paper was recommended for publication by
Editor Jens Kober upon evaluation of the Associate Editor and Reviewers’
comments.
This work was supported by the Natural Sciences and Engineering Research Council of Canada (NSERC)} 
\thanks{The authors are with the Department of Electrical and Computer Engineering, University of Waterloo, Waterloo, ON, Canada. \texttt{$\{$nwilde,a6sadegh,stephen.smith$\}$@uwaterloo.ca} }%
\thanks{Digital Object Identifier (DOI): see top of this page.}%
}

%
%

\markboth{IEEE Robotics and Automation Letters. Preprint Version. Accepted November, 29, 2021}
{Wilde \MakeLowercase{\textit{et al.}}: Learning Submodular Objectives for Team Environmental Monitoring} 

%



\maketitle

\begin{abstract}
In this paper, we study the well-known team orienteering problem where a fleet of robots collects rewards by visiting locations. 
Usually, the rewards are assumed to be known to the robots; however, in applications such as environmental monitoring or scene reconstruction, the rewards are often subjective and specifying them is challenging.
We propose a framework to learn the unknown preferences of the user by presenting alternative solutions to them, and the user provides a ranking on the proposed alternative solutions. We consider the two cases for the user: 1) a deterministic user which provides the optimal ranking for the alternative solutions, and 2) a noisy user which provides the optimal ranking according to an unknown probability distribution. For the deterministic user we propose a framework to minimize a bound on the maximum deviation from the optimal solution, namely regret. We adapt the approach to capture the noisy user and minimize the expected regret. Finally, we demonstrate the importance of learning user preferences and the performance of the proposed methods in an extensive set of experimental results using real world datasets for environmental monitoring problems.

\end{abstract}


%
\IEEEpeerreviewmaketitle

\setcounter{secnumdepth}{2}
\section{Introduction}

Autonomous multi-robots systems find wide-spread acceptance in an increasing number of applications such as persistent monitoring, environmental data collection, shared autonomy and scene reconstruction. A key challenge remains the design of frameworks that allow users who are not robotic experts to deploy them effectively and efficiently.

We study a generalized version of the well known Team-Orienteering Problem (TOP) \cite{OP_survey} where a fleet of robot has to visit multiple locations in the environment. Upon visit, the respective robot collects a reward and the objective is to maximize the total reward collected by the fleet, subject to constraints on the robots' maximum travel distance. 
Multiple variants have been studied, including uncertainty \cite{TOP_cover_uncertain_reward, jorgensen2018team}, and complex reward functions modelling correlations \cite{OP_correlated_Rus} and diminishing returns \cite{submodular_cost_constrained}.

In some applications, such as servicing tasks or delivery, the reward is directly given, e.g., as a monetary value.
However, in other applications the reward can be difficult to quantify and might be user dependent. For example, in environmental monitoring, scientists may have differing opinions on the importance of gaining information in a certain region.
Often the user can indicate regions of interest that the robots should visit. Yet, defining numerical values for a reward function to prioritize between regions is challenging. This is further enhanced when the reward exhibits a diminishing return property: Additional visits of the same region have decreasing additional value. 
Thus, defining reward functions becomes impractical, especially when the user is not a robotics expert.

In human-robot interaction (HRI) the problem of defining reward functions is known as \emph{reward design}.
To reduce the complexity and thus enable a broader range of users to deploy autonomous robots, researchers have studied different frameworks for \emph{reward learning} \cite{dragan_implicitchoice, dragan_orig, IRL_apprentice_learning, sadigh2019, user_study_paper, IROS2020paper, reward_learn_critiques, imitation_Niekum, korein2018multi,jain2015learning, wilson2012bayesian, shah2020interactive}. In contrast to designing parameters of a reward function, users interact with the robot via modalities such as demonstrations, corrections, critique, or choice feedback.

We apply learning from choice to enable users to specify complex submodular reward functions for GTOP and present new solution techniques that are able to handle the high number of dimensions often encountered in these problems.
In our framework the robot fleet is given a set of areas of interest. Over multiple iterations, the user is (virtually) presented with two different sets of tours for the robot fleet; they then choose the preferred option. Using a finite set of submodular basis functions, the user's choice allows the robot fleet to estimate the user reward function.
\begin{figure}[t]
		\centering
		\begin{subfigure}[b]{0.24\textwidth}
            \centering
            \includegraphics[angle=-90, width=\textwidth]{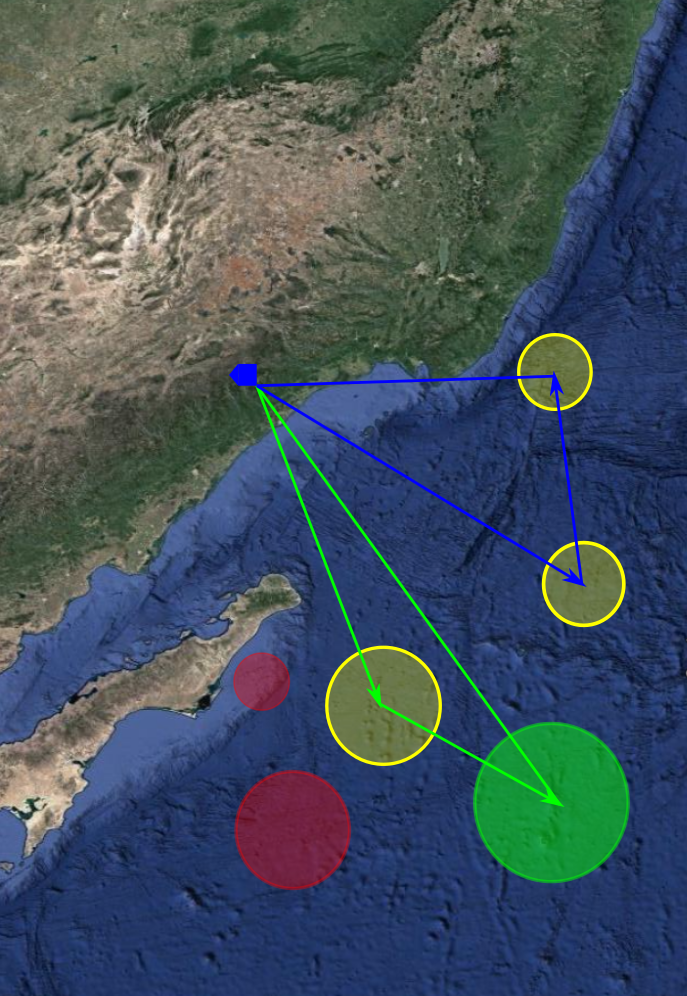}
            \caption{}
        \end{subfigure}%
        \hfill
        \begin{subfigure}[b]{0.24\textwidth}
            \centering
            \includegraphics[angle=-90, width=\textwidth]{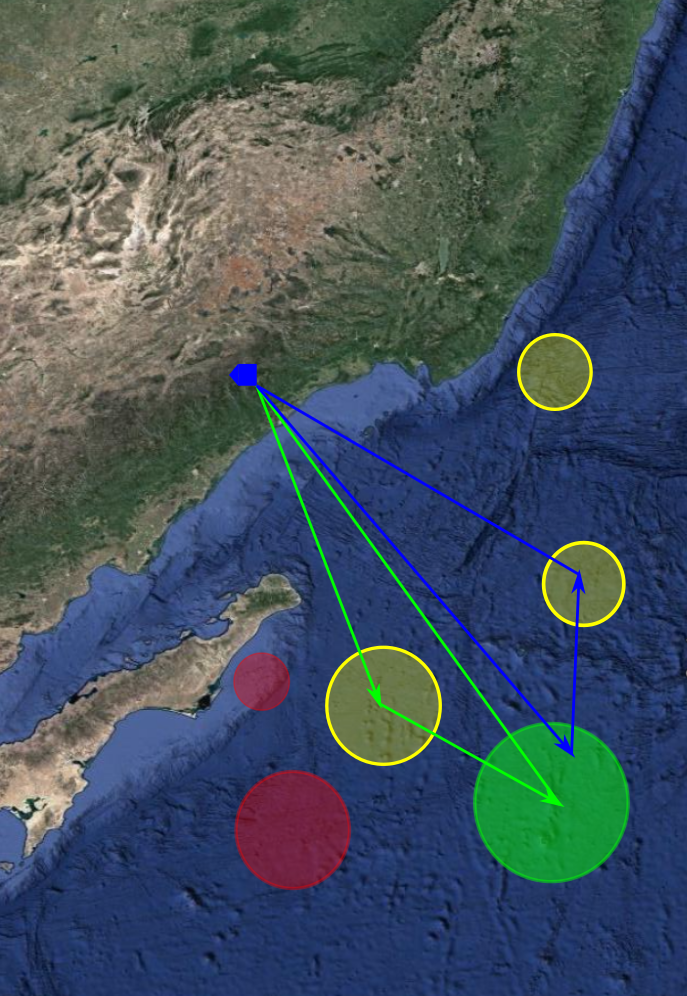}
            \caption{}
        \end{subfigure}%
		
		\caption{Tours for a fleet of two robots for different reward functions with high (green) and low (red) rewards assigned to regions. (a) shows a scenario the robots are not aware of the user preferences over the regions and prioritize the regions equally. In (b) the robots have learned an estimate of the user preference over the regions and identified that re-visiting the green region is more valuable.}
		\label{fig:intro_example}
\end{figure} 
Figure \ref{fig:intro_example} illustrates an example for environmental data collection. The regions of interest are protected areas along a coastline. Without designing or learning a reward function, the robots prioritize them equally (a). Learning from choice feedback allows the robot fleet to identify which regions are most relevant to the user and finds better tours (b).

\textbf{\emph{Contributions:}} In this paper we make the following contributions:
(1) We design submodular basis functions to describe rewards for the generalized team orienteering problem. 
(2) We propose a novel heuristic policy for active preference learning that can handle a high number of basis functions.
(3) To handle uncertainty in the user feedback, we present a novel framework that casts this probabilistic problem to a distribution over instances with noiseless feedback, allowing for efficient learning under uncertainty.
(4) Finally, we demonstrate the practicality of the approach in simulation using real-world locations for environmental data collection.

\textbf{\emph{Related Work:}} 
We address the challenge of defining reward functions for generalized team-orienteering problems and propose an interactive learning framework.
Similarly, researchers in HRI study the design of interactive frameworks that allow inexperienced users to define reward functions for autonomous robots in a wide range of applications.
Since classical approaches such as learning from demonstrations are not always suitable, alternative modes of interaction including corrections, proxy rewards, critique, and choice have been developed  \cite{dragan_implicitchoice}. This work is based on \emph{learning from choice} (sometimes also referred to as active preference learning), where a user iteratively chooses between two presented options \cite{dragan_orig, holladay2016active, wilson2012bayesian, sadigh2019, IROS2020paper, user_study_paper}. Similar to existing work, we pose the problem as learning weights in a linear reward function.
We make novel contributions to address challenges arising from the high dimensionality often found in multi-robot problems. Existing approaches usually rely on sampling potential solutions as well as weights for the reward function \cite{dragan_orig, sadigh2019}. We study how the min-max regret technique from \cite{IROS2020paper} can be extended so it can be used without any samples in a noiseless setting. Further, to handle noisy user feedback we propose a method to cast such feedback to multiple noiseless instances and solve the problem on them.

We focus on a generalized version of the team orienteering problem (TOP) \cite{OP_survey}, which is NP-hard. 
Using basis functions we consider variances of TOP where reward functions can be correlated between vertices, as well as have a diminishing return, i.e., are submodular \cite{nemhauser_submod}.
The authors of \cite{OP_correlated_Rus} study OPs with correlated rewards and give a Mixed-Integer-Quadratic-Program solution, which can also be applied to the multi-robot case.
For single-robot OP with submodular rewards, the authors of \cite{submodular_cost_constrained} provide a constant factor approximation algorithm. The authors of \cite{TOP_sequential_approx} propose an approximation algorithm for TOP by sequentially solving single OPs for each robot. We combine the latter two techniques to obtain a constant factor approximation for submodular TOP.

Stochastic variants of orienteering problems include uncertainty on edge weights \cite{jorgensen2018team, adapt_OP_stoch_traveltime}, time to service a location \cite{OP_stochastic_jobtime}, and rewards \cite{TOP_cover_uncertain_reward}. Similar to the latter case, we study the problem where the robot fleet does not know the rewards.
In \cite{OP_IP_env_monitoring} the rewards of an orienteering problem are dynamic and depend on measurements taken at previous locations. In \cite{TOP_cover_uncertain_reward} robots learn the reward function by iteratively executing tours and use these observations to improve future tours. In contrast, our framework allows robots to learn the reward function by querying a user; the user then does not assign a reward to a single set of tours but instead chooses the preferred set of tours among two presented options. Another difference to \cite{TOP_cover_uncertain_reward} is that our framework can be used as an offline method, where the user is shown tours virtually allowing robots to learn the reward function prior to execution.

Potential applications of the proposed framework include persistent monitoring \cite{smith2011persistent, asghar2019multi}, environmental data collection \cite{OP_IP_env_monitoring}, and scene reconstruction \cite{sadeghi2019minimum}. 
The authors of \cite{HRI_marine_data_Hollinger, HRI_marine_data_Hollinger_ranking} propose an interactive framework for marine data collection: Users define desired targets for observation, the robot then proposes alternatives based on additional information about risks in the environment. Similar to our work, users then choose between different options. This allows the robot to learn the user's utility function trading off reward and risk.

\section{Problem Formulation}

Consider a set of $m$ robots collecting information in an environment represented by a graph $G = (V, E, l)$. The set $V$ is the set of vertices, $E$ is the set of edges between the vertices, and $l:E \rightarrow \mathbb{R}_{\geq0}$ assigns costs to the edges of the graph. A tour $T$ is a sequence of vertices $\langle v_0, v_1, \ldots, v_n, v_0\rangle$. Given a depot location $s \in V$, a tour starts at $s$, i.e., $v_0 = s$. The reward function $R : 2^V \rightarrow \mathbb{R}_{+}$ assigns a reward to each set of vertices, i.e., the reward of visiting the vertices in a tour $T$ denoted by $V(T)$. With a slight abuse of notation we write the reward $R(V(T))$ simply as $R(T)$.

\begin{problem}[Orienteering Problem]
\label{prob:SOP}
Given a graph $G = (V, E, l)$, a reward function $R:2^V \rightarrow \mathbb{R}_{\geq 0}$ and a positive number $B$, find a tour $T$ of length at most $B$ that maximizes the reward collected $R(T)$. 
\end{problem}

We are interested in monotone, normalized submodular reward functions~\cite{nemhauser_submod}. Such a function has the following properties: i) $R(\emptyset) = 0$, ii) $R(A) \leq R(B)$ for every $A \subseteq B \subseteq V$, and iii) $R(A \cup \{v\}) - R(A) \geq R(B \cup \{v\}) - R(B)$  for every $A \subseteq B \subseteq V$ and for every $v \in V$. Now we introduce a generalization of Problem~\ref{prob:SOP} where there are $m$ heterogeneous robots maximizing a submodular reward function.

\begin{problem}[Generalized Team Orienteering Problem (GTOP)]
\label{prob:GTOP}
Consider a graph $G = (V, E, l)$, a fleet of $m$ robots with travel budgets $B_1, \ldots, B_m$, a partition of the vertices into $m$ subsets $V_1, \ldots, V_m$, and a submodular reward function $R: 2^V \rightarrow \mathbb{R}_{\geq 0}$. Find a set of $m$ tours $\mathcal{T} = \{T_1, \ldots, T_m\}$ maximizing the total reward collected  subject to the constraints that $T_i$ for all $i \in \{1, \ldots, m\}$ only visits vertices in $V_i$ and has length at most $B_i$, i.e.,

\be
\begin{aligned}
    \max_{\mathcal{T}} &\quad R(\cup_{i = 1}^{m}V(T_i))\\
     \text{subject to:} & \quad \ell(T_i) \leq B_i,  V(T_i) \subseteq V_i\;\; \forall  i \in \{1, \ldots, m\}\\
\end{aligned}
\label{eq:GTOP}
\ee
\end{problem}

\begin{remark}[Comments on Problem Formulation]
The set $V_i$ contains a vertex for each location that robot $i$ can visit.  Each vertex is contained in just one set. Thus, if multiple robots can visit the same location, then each has a corresponding vertex in their set $V_i$.  The advantage of this representation is that each vertex encodes both the location and the robot performing a visit, and thus a submodular function can be defined directly over sets of vertices.  This is in contrast to other formulations of the submodular team orienteering problem~\cite{jorgensen2018team, xu2020approximation} where the submodular function is defined over the set of all tours/paths, which is exponential in the number of vertices.~\oprocend

\end{remark}

\paragraph{GTOP with Unknown Reward Function}
In this work we consider the case where the reward function $R$ is unknown to the robot. We denote the hidden optimal reward function as $R^*$ and the corresponding GTOP solution $\Tcal^*$. 
Further, let $\hat{R}$ be a robot's estimate of the reward function and $\hat{T}$ the corresponding optimal tour; we are interested in finding an estimate $\hat{R}$ of the reward function with corresponding optimal tours $\hat{\Tcal}=\{\hat{T_1}, \dots, \hat{T_m}\}$ that solves
\be
  \begin{aligned}
       \max_{\hat{\Tcal}} \quad 
       R^*(\hat{\Tcal}) \quad
     \text{subject to:} \ \ell(\hat{T_i}) \leq B_i \text{ for all } \hat{T_i} \in \hat{\T}.
  \end{aligned}
  \label{eq:objective_estimated_reward}
\ee
We notice that this is an ill-posed problem as the reward function $R^*$ is not available to the robot. However, we consider a framework where the robot iteratively interacts with the user, allowing it to make observations about the user's hidden reward function. This is known as \emph{reward learning}, where the robot presents the user with one or multiple possible solutions to it's task and then obtains feedback in the form of corrections, choice, labels, or others~\cite{dragan_implicitchoice}. We define a query as a set of possible solutions for the GTOP $Q=\{\T_1, \T_2, \dots, \T_l\}$. 
Let $f(R)$ be some prior belief over the set of all possible reward functions  $\R$. Given feedback $U$, the robot can compute a posterior $f(R|(Q,U))$. We can express the expected outcome with respect to the prior $\E_{U\sim f(R)}[R|(Q,U)]$. This framework allows us to state our problem as an adaptive stochastic optimization problem:

\begin{problem}[Learning GTOP Rewards]
Given $G=(V,E,l)$, a hidden reward function $R^*$, a fleet of $m$ robots with travel budgets $B_1,\dots, B_m$; find a sequence of $K$ queries $(Q_1, Q_2, \dots, Q_K)$ such that the expected estimated reward function $\hat{R}=\E_{U_k\sim f(R)}[R|(Q_1,U_1), \dots, (Q_K, U_K)]$ and the corresponding sets of $m$ tours $\hat{\Tcal}$ solves~\eqref{eq:objective_estimated_reward}.

\label{prob:learn_problem}
\end{problem}

\section{Team Orienteering Problem with submodular basis functions}
\label{sec:TOP_basis_fn}
In this section, we present a linear approximation of a submodular reward function, then we propose an approximation algorithm for the GTOP for the linearized reward function.
\subsection{Basis functions}
We consider the reward function $R(\T)$ to be be composed of a set of basis functions $r_1, \dots, r_n:2^V\rightarrow \mathbb{R}_{\geq0}$. Given tours $\T$, the reward function is then a weighted sum of the basis functions $R(\T, \w) = \sum_{i=1}^n w_i r_i(\T)$.

Similar approaches are commonly used in reward learning problems \cite{dragan_implicitchoice, dragan_orig, sadigh2019, IROS2020paper,user_study_paper}, where basis functions are usually referred to as \emph{features}.
Given that $r_i$ depend only on the vertices, we can assume without loss of generality that each basis function is characterized by a subset $W_i\subseteq V$. That is, for any $W_i$, let $\psi_i(\T)$ be a count of how many vertices of the tours $\T$ lie in $W_i$, then $r_i(\T)$ is a functional of $\psi_i(\T)$.

The subsets $W_1,\dots, W_n$ can reflect a spatial relation between vertices, i.e., describe neighborhoods, but can also express other features, such as grouping all vertices where the robots can make certain observations. The basic case $r_i(\T)=\psi_i(\T)$ is a modular function describing the number of times subset $W_i$ is visited.
However, many real-world problems exhibit a diminishing return property \cite{submod_marine, submod_sensor_Krause}. In order for $r_i$ to be growing \emph{submodularly} with $\psi_i$, we choose
\be
\label{eq:gamma_decay}
r_i(\T)= \sum_{\alpha=1}^{\psi_i(\T)} \gamma^{(\alpha-1)},
\ee
where $\gamma\in(0,1]$. If $\gamma=1$ we recover the modular case; in the other extreme that $\gamma\to0$, visiting $W_i$ more than once effectively does not yield a larger reward than the first visit.

Using these basis functions, the problem of learning the user reward function $R^*$ becomes one of learning the weights $\w^* = (w^*_1, \ldots, w^*_n)$, i.e., the importance of each basis function, as well as the decay parameters $\gamma^*_1,\dots, \gamma^*_n$ describing the diminishing return. Unfortunately, the proposed reward function is only linear in the weights $\w$, but not in the decays. Therefore, we assume that $\gamma$ comes from a discrete set $\Gamma$.
For each subset $W_i$ we define $|\Gamma|$ basis functions $r_{i,j}$ for $j=1, \dots, {|\Gamma|}$, using the different values $\gamma_j$.
Using this discretization the overall reward function becomes
\be
\blue{
R(\T, \w) = \sum_{i=1}^{n}  \sum_{j = 1}^{|\Gamma|} w_{i,j} r_{i,j}(\mathcal{T}).
}
\label{eq:basis_fn}
\ee
For a sparse notation let $\vphi=\bv r_1, r_2, \ldots\ev$, allowing us to write $R(\T, \w)=\vphi(\T)\cdot\w$. Further, for any given weight $\w$, let $\T(\w)$ denote the set of tours maximizing the reward, i.e., $\mathcal{T}(\w) = \argmax_{\T} R(\T, \w)$. Consequently, $\vphi(\T)$ denotes the features of the tour $\T$.
\begin{observation}[Submodularity]
The  reward function $R(\T, \w)$ proposed in Equation~\eqref{eq:basis_fn} is a normalized, monotone and submodular set function.
\end{observation}
\begin{proof}
Since $\psi_i(\emptyset)=0$ for all $i$, we have $R(\emptyset, \w)=0$; hence, the function is normalized. Further $V(\T')\supseteq V(\T)$ implies $\psi_i(\T')\geq\psi_i(\T)$; adding an additional vertex can only increase the vertex count of any set $W_i$. Hence, $\sum_{\alpha=1}^{\psi_i(\T')} \gamma^{(\alpha-1)}_j\geq \sum_{\alpha=1}^{\psi_i(\T)} \gamma^{(\alpha-1)}_j$ for any $V(\T')\supseteq V(\T)$ and any $\gamma_j\in(0,1]$, making $R$ monotonically increasing.
Finally, consider any vertex $v$ and two sets of tours $\T$ and $\T'$ where $V(\T')\supseteq V(\T)$. Then $\psi_i(\T')\geq \psi_i(\T)$ for all $i$. 
If $v\in W_i$, the marginal gain is $\Delta(\T,v)=r_{i,j}(\T\cup v)- r_{i,j}(\T)=\gamma_j^{\psi_i(\T)+1}$ and $\Delta(\T',v)=r_{i,j}(\T'\cup v)- r_{i,j}(\T')=\gamma_j^{\psi_i(\T')+1}$. Since $\gamma_j\in(0,1]$ and $\psi_i(\T')\geq \psi_i(\T)$, we have $\gamma_j^{\psi_i(\T')+1}\leq \gamma_j^{\psi_i(\T)+1}$, and thus, $\Delta(\T,v)\geq \Delta(\T',v)$.
On the other hand if $v\notin W_i$ then $\psi_i(\T'\cup v)- \psi_i(\T') = \psi_i(\T\cup v)- \psi_i(\T) = 0$ and $\Delta(\T,v)= \Delta(\T',v) = 0$, i.e., adding $v$ does not change the vertex count $\psi_i$ and thus the value of the basis function $r_{i,j}$. 
Since this holds for all $W_i$, we obtain $R(\T'\cup v)- R(\T') \leq R(\T\cup v)- R(\T)$ and $R$ is submodular.
\end{proof}

\subsection{Solving the GTOP for a Given Set of Weights}
In~\cite{submodular_cost_constrained}, authors provide a bi-criterion approximation algorithm for the orienteering problem with submodular rewards, and in~\cite{TOP_sequential_approx} propose an approximation algorithm to extend the results of the orienteering problem to the team orienteering problem. In the rest of the paper, for a given $\w$, we combine these two approaches to achieve a bi-criterion approximation algorithm for the team orienteering problem with submodular reward functions. Algorithm~\ref{alg:approx_alg} shows the proposed approach for the GTOP problem. The algorithm sequentially solves the submodular orienteering problem with the algorithm proposed in~\cite{submodular_cost_constrained} \blue{(line 3)}. Then after iteration $k$, $\psi_i$ is incremented by the number of vertices $T_k$ visits in $W_i$ \blue{(line 5)}. At each iteration, the $\textsc{SingleOP}$ implements the proposed bi-criterion approximation algorithm in~\cite{submodular_cost_constrained} with approximation factor $\eta = 2(1 - \frac{1}{e})^{-1}$ on the collected rewards. Hence, by Theorem 1 in~\cite{TOP_sequential_approx}, Algorithm~\ref{alg:approx_alg} is a $1 + \eta$ approximation algorithm for the GTOP problem.

\begin{algorithm}[t]	
	\DontPrintSemicolon 
	\KwIn{$G=(V,E,l)$, start $s\in V$, \#robots $m$, budget $B$, weights $\w$}
	\KwOut{Tours $\T$}		
	Initialize $\vpsi=\vect{0}$, 
	$\T \leftarrow \emptyset$\\
	\For{$k=1$ to $m$} {
		$T_k \leftarrow \textsc{SingleOP}(G, s, \w, \vpsi)$\label{algline:single_op}\\
		\For{$i=1$ to $n$} {
		    $\psi_i \leftarrow \psi_i + |\blue{W_i} \cap V(T_k)|$\\
		}
		$\T \leftarrow \T \cup \{T_k\}$\\
	}
	\Return{$\T$}
	\caption{Bi-criterion approximation for GTOP}
	\label{alg:approx_alg}
\end{algorithm}

\section{Learning rewards from choice feedback}
One framework for learning reward functions via user interaction that found widespread attention in HRI in recent years is \emph{learning from choice}. Iteratively, the robots present the user with two alternative solutions to some robot planning problems. The user then chooses the preferred option. The user is assumed to make that choice based on some hidden reward function which allows the robot to infer the parameters of that reward function.

\subsection{Deterministic user feedback}
We begin by posing our problem for a deterministic user, whose feedback always follows the assumed cost function. 
Consider that two sets of tours $\T^1$ and $\T^2$ are proposed to the user, and the user indicates their preference. 

\begin{definition}[Deterministic user model]
Given two sets of tours $\T^1$ and $\T^2$, a deterministic user always prefers the tours with larger reward with respect to the hidden user weights $\w^*$. Let $I\in\{1,2\}$ denote the user feedback.Then
\be
R(\T^1,\w^*)\geq R(\T^2,\w^*)
\iff I = 1.
\ee
\end{definition}

\emph{Learning Cuts:} Without loss of generality, assume that the user prefers $\T^1$ (we can simply reassign the labels after observing the choice), therefore we have, $ \vphi(\T^1)\cdot\w -  \vphi(\T^2)\cdot\w \geq 0$. We refer to this inequality as a \emph{cut}. Let $P(c_{1:k})$ denote the polyhedron constructed by the cuts $\{c_1, \ldots, c_k\}$. Now we define a \emph{valid cut} as follows:

\begin{definition}[Valid Cut]
Given a  polyhedron $P(c_{1:k})$, a cut $c_{k + 1}$ is valid if the intersection of the cut and $P(c_{1:k})$ has dimension greater than zero.
\end{definition}

Note that each valid cut partitions the space of valid rewards $\w$, therefore by adding a valid cut at each step we monotonically decrease the set of valid rewards. Now assume that we have a tours $\T^1$ in hand, we want to construct the set $\T^2$ such that the tours in $\T^2$ satisfy the budget constraints and the cut constructed by comparing $\T^1$ and $\T^2$ is valid.

\begin{lemma}
\label{lem:valid_cut}
Given two set of tours $\T^1$ and $\T^2$ and a set of prior cuts $\{c_1, \ldots, c_k\}$, the cut constructed by comparing $\T^1$ and $\T^2$ is valid if and only if the solutions to the following  problems are greater than zero:
\begin{align*}
    \max_{\w} &\ \vphi(\T^1)\w - \vphi(\T^2)\w \quad
    \text{subject to:} \ \w \in P(c_{1:k}),
\end{align*}
\begin{align*}
    \max_{\w} &\ \vphi(\T^2)\w - \vphi(\T^1)\w \quad
    \text{subject to:} \ \w \in P(c_{1:k}).
\end{align*}
\end{lemma}

\begin{proof}
The first part is trivial. Now assume that the cut is valid, then the cut intersects the interior of $P(c_{1:k})$. Let $\vect{d}$ be the vector normal to the line defined by the cut, then there exists a $\w_0$, $\delta_1$ and $\delta_2$ such that $\w_0 + \delta_1 \vect{d}\in P(c_{1:k})$ and  $\w_0 - \delta_2 \vect{d}\in P(c_{1:k})$. Therefore, the solution to the two problems are greater than zero.
\end{proof}

In essence, Lemma~\ref{lem:valid_cut} states that for a valid cut $(\T^1,\T^2)$ there must exist some $\w$ in the current polyhedron $P(c_{1:k})$ for which $\T^1$ has a higher reward than $\T^2$, and vice versa.
\blue{In other words, the hyperplane defining a valid cut passes through the interior of the current polyhedron $P(c_{1:k})$.}

\emph{Query generation:} The main challenge in active preference learning is to iteratively generate valid cuts that allow for efficient learning. 

Related work in HRI is usually based on heuristic solutions that greedily optimize some auxiliary function $h$ to maximize the expected learning benefit of presenting two solutions $(\T^1, \T^2)$. 
Recent approaches include $h$ capturing the volume of the probability space over weights \cite{dragan_orig}, the information entropy \cite{sadigh2019} or the maximum regret \cite{IROS2020paper}.

These optimizations are usually difficult on two different levels: Computing $h$ often poses a hard problem and the potential solutions require solving some robot planning problem, making this a nested optimization. 
Most solution techniques rely on sampling candidates solutions, as well as approximating $h$ using sampled weights.
While this might be suitable for low-dimensional applications, the number of basis functions for team orienteering problems of our problem is $O(n|\Gamma|)$. Thus, accurately approximating information entropy requires a prohibitively large number of samples.

We design a novel query generation method that does not require any form of sampling.
Similar to \cite{user_study_paper} we choose a variation of learning from choice in which one of the two presented options comes from the previous iteration: 
At iteration $k$, let $\Tcurr$ be the tours the user preferred in the previous iteration. We now need to find only one new set of tours $\Tnew$ such that observing feedback to $(\Tcurr, \Tnew)$ yields a valid cut with respect to $\{c_1, \ldots, c_k\}$.

To find $\Tnew$ given the previous cuts $\{c_1,\dots,c_k\}$ and $\Tcurr$, we adapt the maximum regret approach proposed in \cite{IROS2020paper}. Regret measures how suboptimal the solution of estimated parameters ${\w'}$ is. In the GTOP, this is captured by the reward of some tours $\T'$, evaluated by the users true reward function $\w^*$ compared against the user-optimal solution $\T^*$, evaluated by $\w^*$, i.e., 
$
R(\T^*, \w^*) - R(\T', \w^*) 
= \vphi^*\w^* - \vphi'\w^*.
$
Using regret we can find $\Tnew$ by solving
\be
\label{eq:max_regret}
\begin{aligned}
   \max_{\w^{\new}} &\ \vphi(\T(\w^{\new}))\w^{\new} - \vphi(\Tcurr)\w^{\new}\\
\text{subject to:}&\ \w^{\new} \in P(c_{1:k-1}).  
\end{aligned}
\ee

That is, given the current solution, we seek to find $\Tnew$ such that if $\Tnew$ was optimal, the current solution would be most suboptimal, i.e., have maximum regret. If the user chooses $\Tcurr$, the weight $\wnew$ becomes infeasible thus $(\wcurr,\wnew)$ will no longer be the maximizer for the updated polyhedron $P(c_{1:k})$ -- we greedily reduce the upper bound on the error. On the other hand, if the user chooses $\Tnew$, we improve the current solution. \blue{
This formulation makes two major changes to the max regret approach in \cite{IROS2020paper}:
1) we fix one set of tours to be shown to be $\Tcurr$, and 2) we use the difference instead of a ratio in the definition of regret.} \blue{
The following proposition ensures that an algorithm that iteratively solves \eqref{eq:max_regret}  and then updates the polyhedron given the user feedback will eventually find an optimal solution, i.e., a weight $\wcurr$ where $R(\T(\wcurr)) = R(\T(\w^*))$.}

\begin{proposition}
If the optimal solution to Problem~\eqref{eq:max_regret} is not a valid cut, then the reward collected by $\T^{\mathrm{curr}}$ is optimal. 
\end{proposition}

\begin{proof}
Let $\w^{\new}$ be the optimal solution to Problem~\eqref{eq:max_regret}. Since the cut defined by $\Tcurr$ and $\T(\w^{\new})$ is not a valid cut, then we have $\vphi(\T(\w^*))\w^* - \vphi(\Tcurr)\w^* \leq \vphi(\T(\w^{\new}))\w^{\new} - \vphi(\Tcurr)\w^{\new} = 0$,
where the first inequality comes from $\w^* \in P(c_{1: k - 1})$ and the second equality comes from  Lemma~\ref{lem:valid_cut}. Therefore, $\Tcurr$ collects the optimal reward.
\end{proof}

\blue{Now we establish the following result on the complexity of Problem~\eqref{eq:max_regret}.}

\begin{lemma}
\label{lem:complexity}
The problem of finding the tour with maximum regret is NP-hard.
\end{lemma}
\begin{proof}
We show the result by a reduction from the traveling salesman problem (TSP). Given a TSP instance  $G = (V, E, c)$ and a budget $B$, we construct an instance of problem~\eqref{eq:max_regret}. We set the polyhedron $P$ to be unit cube and $T^{curr}$ to be the set of empty tours. Then problem~\eqref{eq:max_regret} becomes $\max_{\w^{\new}}  \vphi(\T(\w^{\new}))\w^{\new}$ and budget $B$ on the tours. Let $\w^*$ be the optimal solution to this problem, then $\T(\w^*)$ where $\vphi(\T(\w^*))\w^* = |V|$ is a valid solution to the TSP. Now note that if there is no solution to the max regret problem collecting $|V|$ reward, then there is no solution to the TSP problem, therefore, the result follows immediately.
\end{proof}

{
We observe the following property of the objective function in~\eqref{eq:max_regret} which will help us provide a bound on it.
}
\begin{lemma}
\label{lem:convexity}
The objective function $\vphi(\T(\w))\w - \vphi(\T')\w$ is a convex function in $\w$ for any set of tours $\T'$.
\end{lemma}
\begin{proof}
Consider $\w = \lambda \w^1 + (1 - \lambda) \w^2$ for some $\lambda \in [0, 1]$. Then, 
\begin{align*}
    \vphi(\T)\w &= \vphi(\T(\w)) [\lambda \w^1 + (1 - \lambda) \w^2] \\
    & =  \lambda \vphi(\T(\w))\w^1 + (1 - \lambda) \vphi(\T(\w))\w^2\\
    & \leq \lambda \vphi(\T(\w^1))\w^1 + (1 - \lambda) \vphi(\T(\w^2))\w^2.
\end{align*}
Note that the second term in the objective function is linear in $\w$. Therefore, the result follows immediately.
\end{proof}

While Lemma~\ref{lem:complexity} shows that finding the set of tours maximizing the regret is NP-hard, Lemma \ref{lem:convexity} implies that the optimal solution of \eqref{eq:max_regret} is on a vertex of the polyhedron $P(c_{1:k})$.
We can upper bound that solution with
\begin{align}
\min_{\w^{\new}} \;& \vphi(\Tcurr)\cdot \w^{\new}\label{problem:upper_bound} \quad
\text{subject to:}\ \w^{\new} \in P(c_{1:k}).
\end{align}

In conclusion, at iteration $k$ our min-max regret heuristic proposes two new sets of tours $(\Tcurr,\Tnew)$ where $\Tcurr$ is the solution the user preferred in the previous iteration, and $\Tnew$ is the approximate GTOP solution for $\wnew$ solving \eqref{problem:upper_bound}.

\subsection{Extension to noisy user feedback}
In the previous section, we considered the problem with a deterministic user who always chooses the set of tours with higher reward with respect to $\w^*$.
In practice, this assumption can lead to suboptimal outcomes when the user decision is not accurately captured in the assumed reward function. Thus, we consider the problem with a noisy user where the set of tours chosen by the user is not the set of tours collecting higher rewards. We model the noisy user with the Boltzmann model as follows:

\begin{definition}[Noisy user model]
\label{def:noisy}
Given two sets of tours $\mathcal{T}^1$ and $\mathcal{T}^2$, and a user with hidden rewards $\w$, then the probability that the user chooses $\mathcal{T}^1$ is 
\[
\mathbb{P}(\mathcal{T}^1, \w) 
= \frac{1}{1+\mathrm{exp}(\beta (\phi(\mathcal{T}^2) - \phi(\mathcal{T}^1))\cdot \w)},
\]

where $\beta\blue{ > 0}$ represents the level of expertise of the user.
\end{definition}

The Boltzmann model is widely used in reward learning \cite{dragan_orig, dragan2, sadigh2019, imitation_Niekum} and describes a user whose choice becomes more uncertain when the presented options have a similar reward with respect to $\w^*$.

Now consider a set of cuts $\{c_1, \ldots, c_k\}$ which are results of the preference questions, then the probability that the hidden reward function lies in $P(c_{1:k})$ is 
    $\mathbb{P}(\w^* \in P(c_{1:k})) = \Pi_{i = 1}^{k} \mathbb{P}(c_i)$,
where $\mathbb{P}(c_i)$ is the probability that the user has responded to the $i$th query correctly. We denote the negation of a cut $c_i$ by $\bar{c}_i$ and the probability of it as  $\mathbb{P}(\bar{c}_i) = 1 - \mathbb{P}(c_i)$.

Algorithm~\ref{alg:noisy} shows the proposed algorithm for learning the reward function of the user. In Line~\ref{algline:Q} of the algorithm, we initialize the set of \blue{observed} cuts to $Q = \emptyset$. 

In Line~\ref{algline:ProbableRegions}, function $\textsc{ProbableRegions}(Q, N)$ takes the current set of cuts $Q = \{c_1, \ldots, c_i\}$ and an integer $N$ as input and returns a set of $N$ polyhedrons.
\blue{Each such polyhedron is constructed as follows: We initialize a set of cuts $Q' = Q$. Then for each cut $c_i\in Q'$, we replace $c_i$ with $\bar{c}_i$ with probability $1-\mathbb{P}(c_i)$. This then defines a new polyhedron $P(Q')$. That is, we sample from the set of all $2^K$ possible combinations of cuts or their negations. Thus, the probability of sampling a set of cuts $Q'$ and thus a polyhedron $P(Q')$ is $\mathbb{P}(\w^* \in P(Q'))$.
Considering multiple polyhedrons allows us to take into account inconsistency in the answers by the user.}

For each of the constructed polyhedrons, function  $\textsc{MaximumRegret}$ solves Problem~\eqref{problem:upper_bound}.
We generate a set of tours for each of the rewards as a candidate sets of tours for the next preference questions. Finally, in Line~\ref{algline:discounted_regret} we evaluate the maximum regret for each polyhedron and discount it by their probabilities. 
The level of expertise for the user and the reward function are not known\blue{, however observe that by Definition~\ref{def:noisy} we have $\mathbb{P}(c_i) > \nicefrac{1}{2}$. Therefore,} we approximate the probabilities of regions as a monotonically decreasing function of the number of negated cuts in the construction of the polyhedron.
Finally, the sets of tours with the highest discounted regret is presented to the user as a new query.    

\begin{algorithm}[t]	
	\DontPrintSemicolon 
	\KwIn{graph $G=(V,E,l)$, start $s\in V$, fleet size $m$, budget $B$, sample budget $N$}
	\KwOut{Tours $\T$}		
	Initialize $\w=\vect{1}$, $Q \leftarrow \emptyset$ \label{algline:Q}\\
	$\T= \textsc{GTOP}(G, \w)$\\
	\For{$i=k$ to $K$} {
		$\{P_1, \ldots, P_N\} \leftarrow  \textsc{ProbableRegions}(Q, N)$\label{algline:ProbableRegions}\\
		\For{$P_j \in \{P_1, \ldots, P_N\}$}
		{
		$\w^j \leftarrow \textsc{MaximumRegret}(G, \T, P^j)$\\
	    $\T^j \leftarrow \textsc{GTOP}(G, \w^j)$\\
		}
	$\T_{\mathrm{new}} \leftarrow \underset{\T^j}{\argmax} \ \mathbb{P}(\w^* \in P_j)(\vphi(\T^j) - \vphi(\T)) \cdot \w^j$\label{algline:discounted_regret}\\
	$\T, c_k \leftarrow \textsc{UserResponse}(\T, \Tnew)$\\
	$Q \leftarrow Q \cup \{c_k\}$\\
	}
	\Return{$\T$}   
	\caption{Learning GTOP Rewards}
	\label{alg:noisy}
\end{algorithm}

\section{Evaluation}
\label{sec:evaluation}
We evaluate the performance in environmental monitoring missions using real-world and randomly generated scenarios.
\begin{figure*}[t]
		\centering
		\begin{subfigure}[b]{0.29\textwidth}
            \centering            \includegraphics[width=\textwidth]{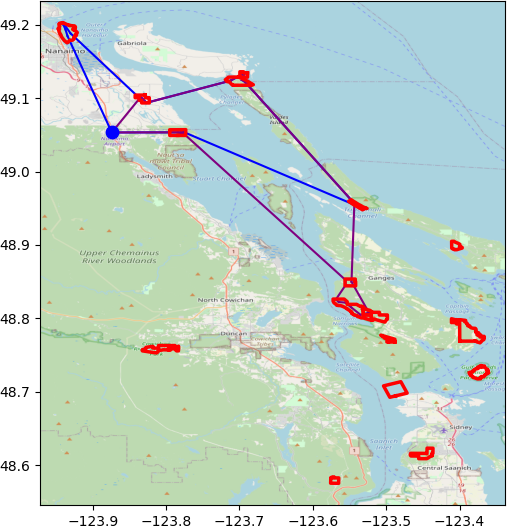}
        \end{subfigure}%
        \hfill
        \begin{subfigure}[b]{0.29\textwidth}
            \centering
            \includegraphics[width=\textwidth]{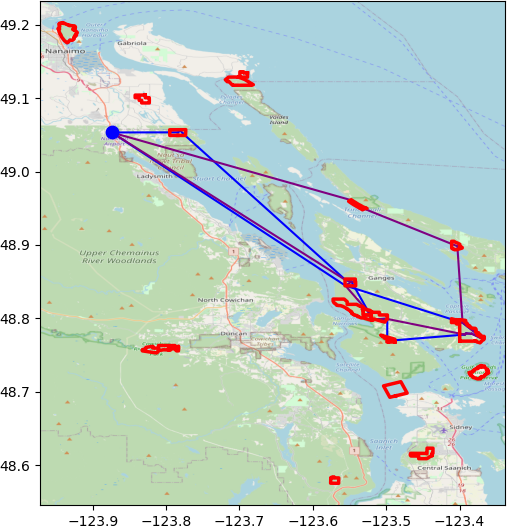}
        \end{subfigure}%
        \hfill
        \begin{subfigure}[b]{0.29\textwidth}
            \centering
            \includegraphics[width=\textwidth]{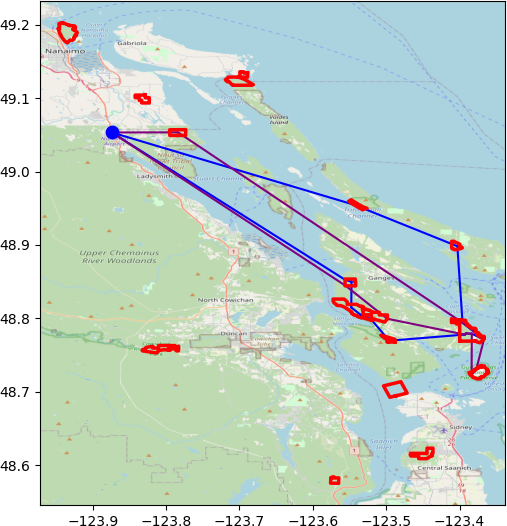}
        \end{subfigure}%
		
		\caption{Example tours from a common depot (blue dot) for the environmental monitoring task. The left figure shows the initial solution before learning with $\w=\vect{1}$. The middle figure shows the tours that were learned after $20$ iterations, while the right figure shows the approximately optimal tours $\T^*$.}
		\label{fig:geo_data_example}
\end{figure*} 
\subsection{Experiment Setup}

In the experiments, the robot fleet is given a set of regions of interest, but no information on how valuable it is to collect environmental data in each region. The objective is to learn a reward function describing which regions are to be prioritized when battery life does not allow to visit all of them.

\paragraph{Basis functions}
We define basis functions that capture the visits to each of the regions of interest.
In each region, we randomly place $1$ to $5$ vertices. For each set $W_i$ of vertices in a single region we define three different basis functions as in \eqref{eq:gamma_decay} for decay parameters $\gamma\in\{.001, .5, 1\}$. That is, the user reward for each region can follow a step function, a curved submodular function, or a linear function.

\paragraph{User design}
Drawing user weights for the basis functions uniformly random does not yield relevant problem instances: the initial solution $\T=\mathtt{GTOP(G,\w=\vect{1}})$ is often already close to the optimal solution.

Thus, we design a probability function for how a user places weights on these basis functions.
First, we model the user weight $w_i$ as a function of the distance of the region $V_i$ from the depot. In detail, the weight is a Gaussian random variable  $w_i=\mathcal{N}(d(s, V_i)^2, \sigma)$ where $d(s, V_i)$ is the distance from the start vertex to the mean location of $V_i$, and the variance $\sigma$ is a design parameter
A second parameter $\vect{p}$ describes a probability over the different values of $\gamma$---for each region the user ``picks'' only one of the three decay parameters. Thus, the vector $\vect{p}=\bv p_1,p_2,p_3\ev$ where $p_1+p_2+p_3=1$ describes the probabilities for $\gamma$ taking values $\{.001, .5, 1\}$, respectively.

We simulate users with values $\sigma\in\{.5, 10\}$ representing scenarios with moderate and almost no correlation of reward and distance, and distributions over decay functions $\vect{p} \in \{[\nicefrac{1}{3}, \nicefrac{1}{3}, \nicefrac{1}{3}], [.7, .2, .1] , [.1, .2, .7], [.2, .7, .1] \}$ where the weights represent a bias towards step, linear and submodular functions, respectively.
\blue{Thus, our simulated users vary in the structure of what regions are important to them, as well as in the type of observations they are interested in, i.e., repeated observations of the same region, or rather covering more regions with fewer or even just one observation.}
For all $8$ \blue{different parameter settings}, we choose $\beta = 20$. This results in users choosing the set of tours with higher rewards in $84\%$ of uniformly random queries.

We consider all robots start at the same depot and have a budget equal to twice the distance from the depot to the furthest region. In the implementation of Algorithm \ref{alg:noisy}, we use a simple static probability $\mathbb{P}(c_i)=0.8$ and $N = 10$ $\textsc{ProbableRegions}$. We measure how well Algorithm \ref{alg:noisy} learns the weights $\w^*$ of the user reward function, i.e., solves Problem \ref{prob:learn_problem}. Let $\T$ be the set of tours returned by Algorithm \ref{alg:noisy}, the reward  ratio is
$
R(\T,\w^*) /R(\T^*,\w^*),
$
where $\w^*$ are the hidden user weights, and $\T^*$ is the corresponding approximate GTOP solution using Algorithm~\ref{alg:approx_alg}.
We notice that $\T^*$ is only an approximation to the optimum; thus, the ratio can be larger than $1$.

\subsection{Baselines}
We compare the proposed maximum regret heuristic against two classes of baselines. 

\paragraph{Richer user input}
The first class of baselines consists of non-learning approaches. Instead, we consider that the user provides more information about their reward function to the robot. In the proposed framework, the user only identifies regions of interest such as the protected areas in Figure \ref{fig:intro_example}. Without any learning, the robot fleet assumes equal reward for all regions, and that the decay can be either linear, submodular, or a step function. This constitutes the initial solution of our algorithm.
With increasing complexity of user input we consider: $\mathtt{Decay}$ -- users do not provide numerical reward values for the regions, but indicate what decay function each region has, $\mathtt{Ranking+Decay}$ -- the user gives a ranking of the importance of each region, and the decay function, $\mathtt{Reward}$ -- the user provides the exact rewards for the region, but not the decay, and finally $\mathtt{Reward+Decay}$ where the user provides the rewards and the decay, which is equivalent to providing $\w^*$. The latter two methods are effectively reward designs, which require a very high level of expertise from the user and thus are often impractical.

\emph{Competing active learning methods:}
The second class of baselines consists of competing approaches for the active query generation.
The first is $\mathtt{Random\, Uniform}$ where we replace line 9 of Algorithm \ref{alg:noisy} with computing $\Tnew$ for some uniformly randomly sampled weight $\wnew$. 
A second method $\mathtt{Random\, Posterior}$ samples based on previously observed user feedback. We replace the direction $\vphi(\Tcurr)$ of the objective function in \eqref{problem:upper_bound} with some uniformly randomly sampled direction $\vect{d}\in [-1,1]^n$, and set $N=1$ in line 5 of Algorithm \ref{alg:noisy}, i.e., generate only one probable region $P$.

The third query generation method $\mathtt{Information\, Gain}$ adapts the information entropy approach proposed in \cite{sadigh2019}. However, the original algorithm is unsuccessful in our problem: It returns the expected weight, which is often rendered $0$ as the weight samples do insufficiently cover the high dimensional space. Thus, we adapted the entropy approach to also follow our framework where we show the previously preferred option $\Tcurr$ again. To find the second set of tours, we compute a set of candidates $\T_1,\dots, \T_k$: We execute lines 5 to 8 of our Algorithm \ref{alg:noisy}, but replace line 7 with a linear program optimizing in a random direction, similar to $\mathtt{Random\, Posterior}$. We then select the best candidate by solving Equation (4) from \cite{sadigh2019}, using $M=300$ weight samples. Each sample is drawn uniformly random from a sampled polyhedron returned by $\textsc{ProbableRegions}(Q, 1)$.

\subsection{Results}
\begin{figure}[h]
		\centering
		\begin{subfigure}[b]{0.49\textwidth}
            \centering
            \includegraphics[width=.9\textwidth]{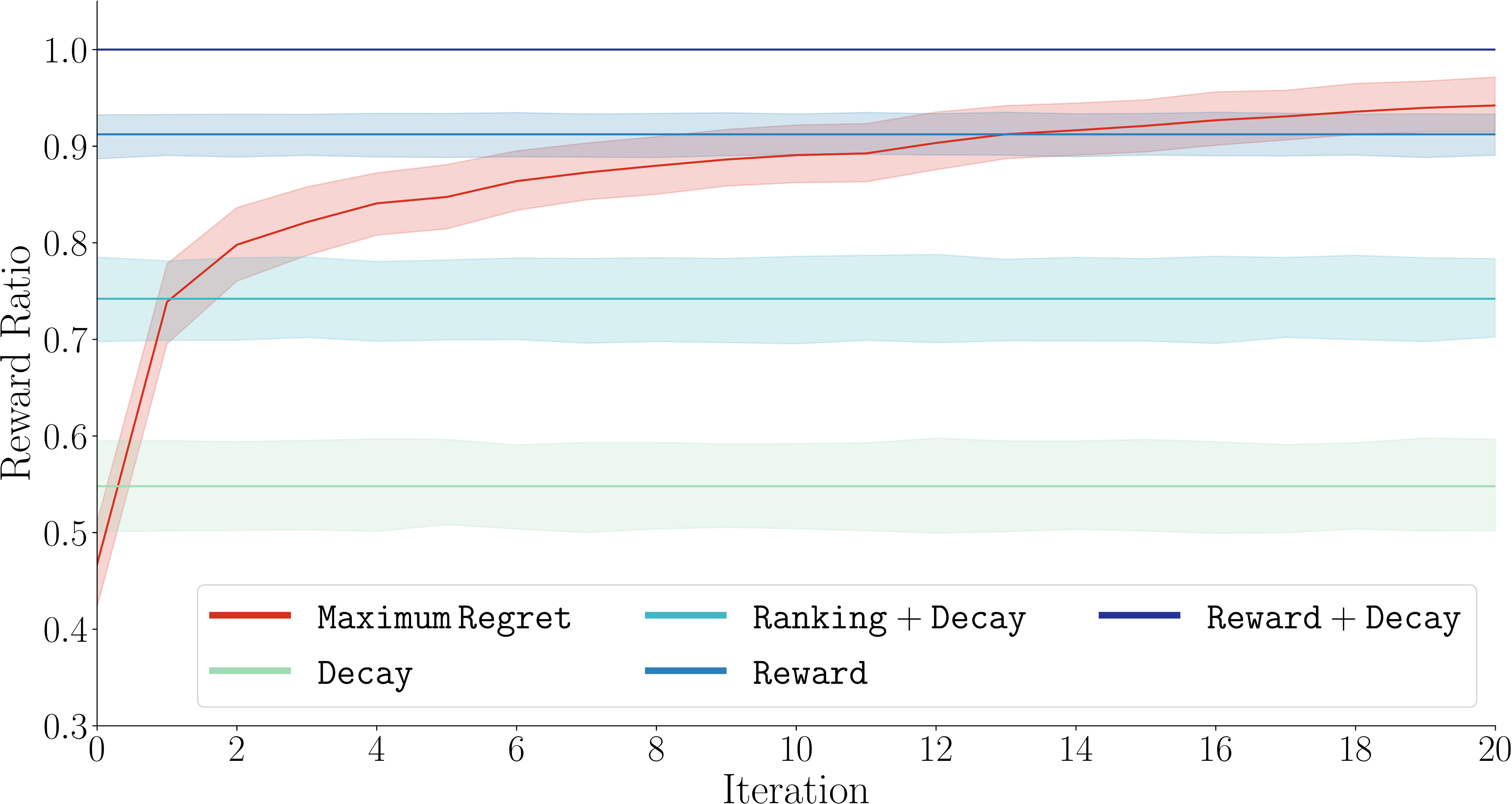}
            \caption{Comparison with richer user input}
    		\label{fig:GeoRegion_GTOP_baselines}
        \end{subfigure}
        
        \begin{subfigure}[b]{0.49\textwidth}
            \centering
            \includegraphics[width=.9\textwidth]{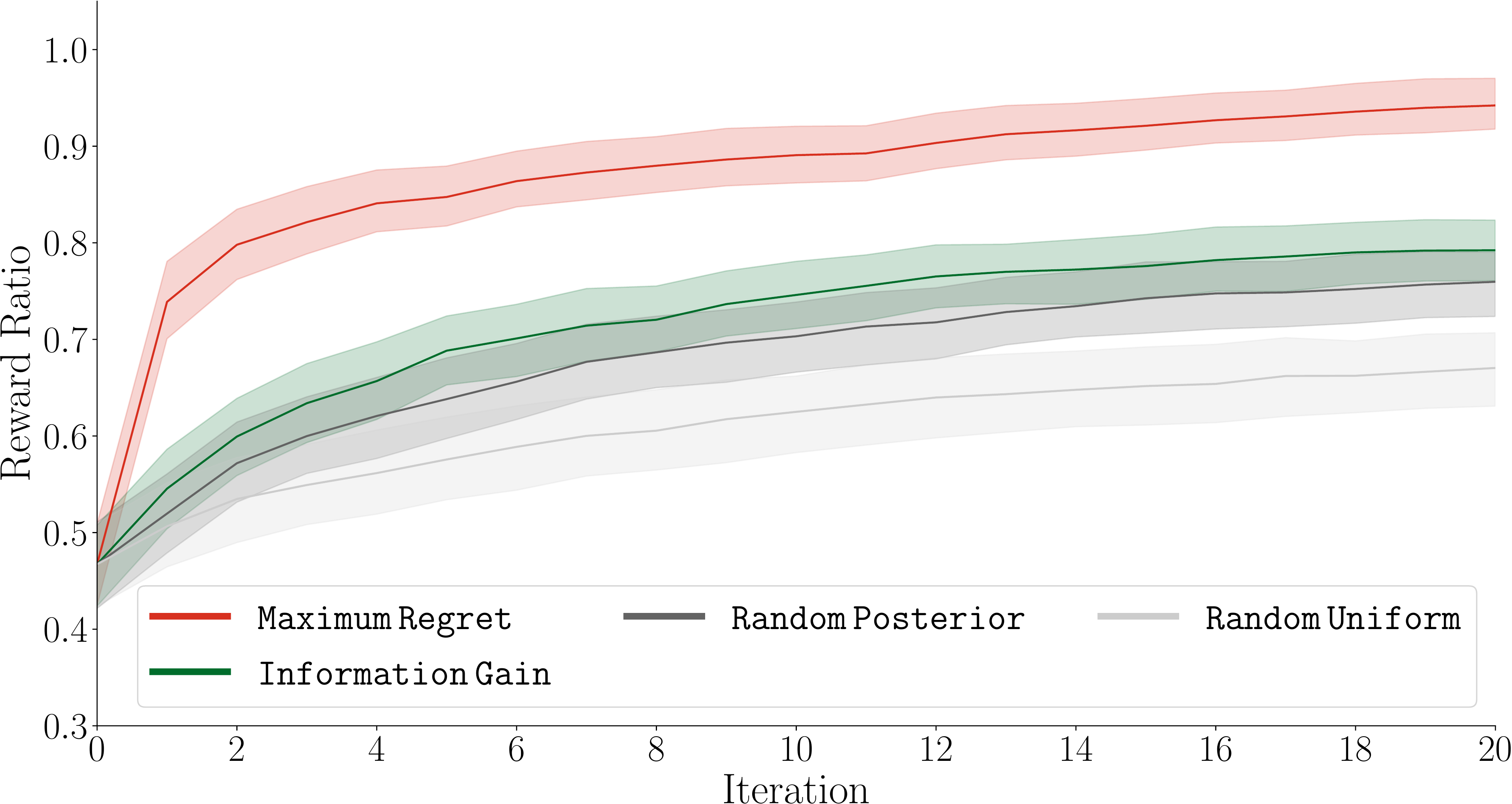}
            \caption{Comparison with different learning methods.}
    		\label{fig:GeoRegion_GTOP_learners}
        \end{subfigure}
        
		\caption{Summarized results for the real world experiment comparing the proposed approach against levels of a-priori user input (a), and against other active preference learning methods (b).}
		\label{fig:Geo_Region_GTOP}
\end{figure}

\paragraph{Real world environment with two robots}
In the first experiment, we use part of the World Database on Protected Areas (WDPA)\footnote{Dataset from https://data.unep-wcmc.org/datasets/12}, illustrated in Figure \ref{fig:geo_data_example}.
The problem contains $17$ regions, and for each of the $8$ user types we simulated 25 trials. 

Figure \ref{fig:Geo_Region_GTOP} shows the comparison of our proposed method against the two baseline classes. 
In Figure \ref{fig:GeoRegion_GTOP_baselines}  we observe that $\mathtt{Maximum\, Regret}$ drastically improves the reward ratio within the first few iterations, finding tours as good as $\mathtt{Ranking+Decay}$, where the user would provide a noiseless ranking of all regions and identify the decay function. After {$12$} iterations the learning collects as much reward as $\mathtt{Reward}$ and approaches $\mathtt{Reward+Decay}$ (i.e., the ground truth) after $20$ iterations. This showcases that the learning approach is able to find tours equally good to those that require much more complex user input when not using preference learning. Furthermore, after $20$ iterations the learned set of tours collects $95\%$ of the reward of the approximately optimal tours.
Figure \ref{fig:GeoRegion_GTOP_learners} shows that $\mathtt{Maximum\, Regret}$ collects the most reward after $20$ iterations ($95\%$ compared to $82\%$ for $\mathtt{Information\, Gain}$, $75\%$ for $\mathtt{Random\, Posterior}$ and $68\%$ for $\mathtt{Random\, Uniform}$). Moreover, there is a strong difference in the learning speed: After only $2$ iterations $\mathtt{Maximum\, Regret}$ already achieves $82\%$, matching the result for $\mathtt{Information\, Gain}$ after $18$ iterations.

\paragraph{Synthetic environment with 4 robots}
To assert that the real-world experiment is representative of a wider range of problem instances, we consider a synthetic experiment where the regions of interest are randomly generated. 
We construct additional problem instances by sampling $20$ regions of interest, each offering $1$ to $5$ observations points. 
We observed similar performance on the synthetic instances: After $10$ iterations the proposed method collects as much reward as $\mathtt{Ranking+Decay}$ and $\mathtt{Reward}$ ($81\%$). After $20$ iterations our method achieves a reward ratio of $89\%$ compared to $72\%$ for
$\mathtt{Random\, Posterior}$, $70\%$ for $\mathtt{Information\, Gain}$, and $68\%$ for $\mathtt{Random\, Uniform}$.

In summary, the proposed method finds high-quality tours for the data collection task that, when no learning was used, would require much richer user input; and our method outperforms other learning approaches in terms of collected reward and learning speed.

\section{Discussion}
\label{sec:discussion}

This paper considers the problem of reward collection by a team of robots with a hidden submodular reward function. First, we presented a linear approximation of the submodular reward function. Second, we proposed an approximation algorithm when the weights on the linear approximation is known. Finally, we proposed a framework to learn the underlying hidden user reward function for deterministic and noisy users from choice feedback. The experimental results on real-world data show that the proposed framework provides near-optimal tours after a few iterations of user feedback. For future work, we consider investigating the effectiveness of the proposed method in data collection missions in the field.

\bibliographystyle{IEEEtran}

\end{document}